%% file: JMLR18.tex
\documentclass[twoside,11pt]{article}

\usepackage{jmlr2e}

\usepackage{algorithm,algorithmic}
\usepackage{amsmath, amssymb}
\usepackage{graphicx}
\usepackage{caption}
\usepackage{subcaption}
\usepackage{xcolor}

\newcommand{\mb}{\mathbb}
\newcommand{\mc}{\mathcal}
\newcommand{\rar}{\rightarrow}

\newcommand{\D}{D}
\newcommand{\bmat}[1]{\begin{bmatrix}#1\end{bmatrix}}
\newtheorem{assumption}{Assumption}
\newcommand{\spec}{\mathrm{spec}}

\newcommand{\f}{f}

\usepackage{multirow}
 
 \usepackage{enumitem}

\usepackage{times}

\usepackage{amsmath, amssymb}
\usepackage{graphicx}
\usepackage{caption}
\usepackage{subcaption}
\usepackage{xcolor}
\usepackage{multirow}

\ShortHeadings{On Finding Local Nash Equilibria (and Only Local Nash Equilibria) in Zero-Sum Games}{Mazumdar, Jordan, and Sastry}
\firstpageno{1}

\begin{document}

\title{On Finding Local Nash Equilibria (and Only Local Nash Equilibria) in Zero-Sum Continuous Games}
\author{\name Eric Mazumdar \email emazumdar@eecs.berkeley.edu \\ \addr Department of Electrical Engineering and Computer Science\\
       University of California, Berkeley, CA 94720-1776, USA
       \AND
       \name Michael I.\ Jordan \email jordan@cs.berkeley.edu \\ \addr Division of Computer Science and Department of Statistics\\
       University of California, Berkeley, CA 94720-1776, USA
       \AND
       \name S. Shankar Sastry \email sastry@coe.berkeley.edu \\ \addr Department of Electrical Engineering and Computer Science\\
       University of California, Berkeley, CA 94720-1776, USA}

\editor{}

\maketitle

\begin{abstract}%
\input{abstract}
\end{abstract}

\section{Introduction}
\input{intro}

\section{Preliminaries}
\label{sec:prelims}
\input{prelims}

\section{Constructing the limiting differential equation}
\input{alg}

\label{sec:alg}

\section{Two-timescale approximation}
\label{sec:thms}
\input{longterm}

% \subsection{Finite time behavior of the two-timescale process}
% \input{conc_bnd}

\section{Numerical Examples}
\label{sec:num}
\input{num_ex}

\section{Discussion}
\label{discussion}
\input{discussion}

\bibliography{JMLR18}

\appendix
\section{Time-varying adjustment}
\label{app:tv}
\input{time_varying}

\section{Counter-examples for other algorithms}
\label{app:counter}
\input{counter_example}

\section{Efficient computation of the two-timescale update}
\label{app:jac_vec_prods}
\input{jvps}

\section{Further numerical examples} 
\label{app:numex}
\input{app_numex}

\section{Useful propositions} 
\label{app:props}
\input{old_props}

% \section{Proof of Concentration Bound}
% \label{app:conc_pf}
% \input{thm}

\end{document}

%% file: abstract.tex
% !TEX root = JMLR18.tex

We propose local symplectic surgery, a two-timescale procedure for finding local Nash equilibria in two-player zero-sum games. We first show that previous gradient-based algorithms cannot guarantee convergence to local Nash equilibria due to the existence of non-Nash stationary points. By taking advantage of the differential structure of the game, we construct an algorithm for which the local Nash equilibria are the only attracting fixed points. We also show that the algorithm exhibits no oscillatory behaviors in neighborhoods of equilibria and show that it has the same per-iteration complexity as other recently proposed algorithms. We conclude by validating the algorithm on two numerical examples: a toy example with multiple Nash equilibria and a non-Nash equilibrium, and the training of a small generative adversarial network (GAN).

%and a similar rate of convergence in regions of attractions of local Nash equilibria. To show the latter, we derive novel lower bounds on the lock-in probability of two time-scale processes under more relaxed assumptions than currently exist in the literature. 

%% file: intro.tex
% !TEX root = JMLR18.tex

The classical problem of finding Nash equilibria in multi-player games has been
a focus of intense research in computer science, control theory, economics
and mathematics~\citep{BasarOlsder,NisanEtAl,NashPPAD}. Some connections have been made between this
extensive literature and machine learning~\citep[see, e.g.,][]{Cesa-BianchiLugosi,
banerjee:2003aa,LOLA}, but these connections have focused principally on
decision-making by single agents and multiple agents, and not on the burgeoning
pattern-recognition side of machine learning, with its focus on large data sets and
simple gradient-based algorithms for prediction and inference.  This gap has begun to
close in recent years, due to new formulations of learning problems as involving competition
between subsystems that are construed as adversaries~\citep{goodfellow:2014ab},
the need to robustify learning systems with regard to against actual
adversaries~\citep{Robustness} and with regard to mismatch between assumptions and data-generating mechanisms~\citep{Drift, GiordanoEtAl}, and an increasing awareness that real-world machine-learning systems are often embedded in larger economic systems or
networks~\citep{Jordan-AI}.

These emerging connections bring significant algorithmic and conceptual challenges
to the fore.  Indeed, while gradient-based learning has been a major success
in machine learning, both in theory and in practice, work on gradient-based
algorithms in game theory has often highlighted their limitations.  For example,
gradient-based approaches are known to be difficult to tune and
train~\citep{daskalakis:2017aa,paper:CO,hommes:2012aa, paper:ACO}, and recent work has shown that gradient-based learning will almost surely avoid a
subset of the local Nash equilibria in general-sum games~\citep{paper:old}.
Moreover, there is no shortage of work showing that gradient-based algorithms
can converge to limit cycles or even diverge in game-theoretic settings~\citep{benaim:1999ac,hommes:2012aa,
daskalakis:2017aa, mertikopoulos:2018aa}.

These drawbacks have led to a renewed interest in approaches to finding the
Nash equilibria of zero-sum games, or equivalently, to solving saddle point
problems.  Recent work has attempted to use second-order information to reduce
oscillations around equilibria and speed up convergence to fixed points of
the gradient dynamics \citep{paper:CO,paper:ACO}.  Other recent approaches have
attempted to tackle the problem from the variational inequality perspective
but also with an eye on reducing oscillatory behaviors \citep{VIP,VIP2}.

None of these approaches, however, address a fundamental issue that arises in
zero-sum games. As we will discuss, the set of attracting fixed points for the
gradient dynamics in zero-sum games can include critical points that are not
Nash equilibria.  In fact, any saddle point of the underlying function that does
not satisfy a particular alignment condition of a Nash equilibrium is a candidate
attracting equilibrium for the gradient dynamics. Further, as we show, these
points are attracting for a variety of recently proposed adjustments to
gradient-based algorithms, including consensus optimization \citep{paper:CO},
the symplectic gradient adjustment \citep{paper:ACO}, and a two-timescale
version of simultaneous gradient descent \citep{paper:wrongTTS}.  Moreover,
we show by counterexample that these algorithms can all converge to non-Nash
stationary points.

We present a new gradient-based algorithm for finding the local Nash equilibria of
two-player zero-sum games and prove that the only stationary points to which the
algorithm can converge are local Nash equilibria.  Our algorithm makes essential use
of the underlying structure of zero-sum games.  To obtain our theoretical results we
work in continuous time---via an ordinary differential equation (ODE)---and our
algorithm is obtained via a discretization of the ODE.  While a naive discretization
would require a matrix inversion and would be computationally burdensome, our
discretization is a two-timescale discretization that avoids matrix inversion
entirely and is of a similar computational complexity as that of other gradient-based
algorithms.

The paper is organized as follows.  In Section~\ref{sec:prelims} we define our
notation and the problem we address.  In Section~\ref{sec:alg} we define the
limiting ODE that we would like our algorithm to follow and show that it has
the desirable property that its only limit points are local Nash equilibria of
the game. In Section~\ref{sec:thms} we introduce local symplectic surgery, a two-timescale procedure that
asymptotically tracks the limiting ODE and show that it can be implemented efficiently.
Finally, in Section~\ref{sec:num} we present two numerical examples to validate the
algorithm. The first is a toy example with three local Nash equilibria, and one
non-Nash fixed point. We show that simultaneous gradient descent and other recently
proposed algorithms for zero-sum games can converge to any of the four points
while the proposed algorithm only converges to the local Nash equilibria.
The second example is a small generative adversarial network (GAN), where
we show that the proposed algorithm converges to a suitable solution within
a similar number of steps as simultaneous gradient descent.
%The first is a toy example with three local Nash equilibria, and one
% non-Nash fixed point. We show that simultaneous gradient descent and other recently
% proposed algorithms for zero-sum games can converge to any of the four points
% while the proposed algorithm only converges to the local Nash equilibria.
% The second example is a small generative adversarial network (GAN), where
% we show that the proposed algorithm converges to a suitable solution within
% a similar number of steps as simultaneous gradient descent.

%% file: prelims.tex
% !TEX root = JMLR18.tex

We consider a two-player game, in which one player tries to minimize a function, $f:\mathbb{R}^d\rar\mathbb{R}$, with respect to their decision variable $x \in \mathbb{R}^{d_x}$, and the other player aims to maximize $f$ with respect to their decision variable $y \in \mathbb{R}^{d_y}$, where $d=d_y+d_x$. We write such a game as $\mc{G}=\{(f,-f),\mathbb{R}^d\}$, since the second player can be seen as minimizing $-f$. We assume that neither player knows anything about the critical points of $f$, but that both players follow the rules of the game. Such a situation arises naturally when training machine learning algorithms (e.g., training generative adversarial networks or in multi-agent reinforcement learning). Without restricting $f$, and assuming both players are non-cooperative, the best they can hope to achieve is a \emph{local Nash equilibrium}; i.e., a point $(x^*,y^*)$ that satisfies
\[f(x^*,y)\le f(x^*,y^*) \le f(x,y^*),\]
for all $x$ and $y$ in neighborhoods of $x^*$ and $y^*$ respectively. Such equilibria are locally optimal for both players with respect to their own decision variable, meaning that neither player has an incentive to unilaterally deviate from such a point. As was shown in \citet{ratliff:2013aa}, generically, local Nash equilibria will satisfy slightly stronger conditions, namely they will be differential Nash equilibria (DNE):

\begin{definition}
  \label{def:DNE}
  A strategy $(x^*,y^*) \in \mathbb{R}^n$ is a \emph{differential Nash equilibrium} if:
  \begin{itemize} 
    \item $D_xf(x^*,y^*)=0$ and $D_yf(x^*,y^*)=0$.
    \item $\D^2_{xx}\f(x^*,y^*)\succ 0$, and $\D^2_{yy}\f(x^*,y^*)\prec0$.
  \end{itemize} 
\end{definition}
Here $D_xf$ and $D_yf$ denote the partial derivatives of $f$ with respect to $x$ and $y$ respectively, and $D^2_{xx}f$ and $D^2_{yy}f$ denote the matrices of second derivatives of $f$ with respect to $x$ and $y$. Both differential and local Nash equilibria in two-player zero-sum games are, by definition, special saddle points of the function $f$ that satisfy a particular \emph{alignment condition} with respect to the player's decision variables. Indeed, the definition of differential Nash equilibria, which holds for almost all local Nash equilibria in a formal mathematical sense, makes this condition clear: the directions of positive and negative curvature of the function $f$ at a local Nash equilibria must be \emph{aligned} with the minimizing and maximizing player's decision variables respectively.

We note that the key difference between local and differential Nash equilibria is that $\D^2_{xx}\f(x,y)$, and $\D^2_{yy}\f(x,y)$ are required to be definite instead of semidefinite. This distinction simplifies our analysis while still allowing our results to hold for almost all continuous games.

\subsection{Issues with gradient-based algorithms in zero-sum games}
Having introduced local Nash equilibria as the solution concept of interest, we now consider how to find such solutions, and in particular we highlight some issues with gradient-based algorithms in zero-sum continuous games. The most common method of finding local Nash equilibria in such games is to have both players randomly initialize their variables $(x_0,y_0)$ and then follow their respective gradients. That is, at each step $n=1,2,...$, each agent updates their variable as follows:
\begin{align*}
x_{n+1}&=x_n-\gamma_nD_xf(x_n,y_n)\\
y_{n+1}&=y_n+\gamma_nD_yf(x_n,y_n),
\end{align*}
where $\{\gamma_n\}_{n=0}^\infty$ is a sequence of step sizes. The minimizing player performs gradient descent on their cost while the maximizing player ascends their gradient. We refer to this algorithm as simultaneous gradient descent (simGD). To simplify the notation, we let $z=(x,y)$, and define the vector-valued function $\omega: \mathbb{R}^d \rar \mathbb{R}^d$ as:
\[ \omega(z)=\begin{bmatrix} D_x f(x,y) \\ -D_y f(x,y) \end{bmatrix}. \]
In this notation, the simGD update is given by:
\begin{align}
    z_{n+1}=z_n-\gamma_n\omega(z_n).
    \label{eq:dtomega}
\end{align}
Since \eqref{eq:dtomega} is in the form of a discrete-time dynamical system, it is natural to examine its limiting behavior through the lens of dynamical systems theory. Intuitively, given a properly chosen sequence of step sizes, ~\eqref{eq:dtomega} should have the same limiting behavior as the continuous-time flow:
\begin{align}
    \dot z=-\omega(z).
    \label{eq:ctomega}
\end{align}
We can analyze this flow in neighborhoods of equilibria by studying the Jacobian matrix of $\omega$, denoted  $J: \mathbb{R}^d \rar \mathbb{R}^{d \times d}$:
\begin{align}
J(z)=\begin{bmatrix} \ \ \ D^2_{xx} f(x,y) & \ \ \ D^2_{yx} f(x,y) \\ -D^2_{xy} f(x,y) &-D^2_{yy} f(x,y) \end{bmatrix}.
\label{eq:jac}
\end{align}
We remark that the diagonal blocks of $J(z)$ are always symmetric and $D^2_{xy}f=(D^2_{yx}f)^T$. Thus $J(z)$ can be written as the sum of a block symmetric matrix $S(z)$ and a block anti-symmetric matrix $A(z)$, where:
\begin{align*}
S(z)=\begin{bmatrix} \ \ \ D^2_{xx} f(z) & \ \ \ 0 \\ 0 &-D^2_{yy} f(z) \end{bmatrix} \ \ ; \ \ A(z)=\begin{bmatrix} \ \ \ 0 & \ \ \ D^2_{yx} f(z) \\ -D^2_{xy} f(z) &0 \end{bmatrix}.
\end{align*}

Given the structure of the Jacobian, we can now draw links between differential Nash equilibria and equilibrium concepts in dynamical systems theory. We focus on hyperbolic critical points of $\omega$.

\begin{definition}
A strategy $z\in \mb{R}^d$ is a critical point of $\omega$ if $\omega(z)=0$. It is a \emph{hyperbolic critical point} if $Re(\lambda)\ne 0$ for all $\lambda\in\spec(J(z^*))$, where $Re(\lambda)$, denotes the real part of the eigenvalue $\lambda$ of $J(z^*)$.
\end{definition}
It is well known that hyperbolic critical points are generic among critical points of smooth dynamical systems (see e.g. \citep{sastry:1999aa}), meaning that our focus on hyperbolic critical points is not very restrictive. Of particular interest are locally asymptotically stable equilibria of the dynamics.

\begin{definition}
A strategy $z^*\in \mb{R}^d$ is a \emph{locally asymptotically stable equilibrium} (LASE) of the continuous-time dynamics $\dot z=-\omega(z)$ if $\omega(z^*)=0$ and $Re(\lambda)>0$ for all $\lambda\in\spec(J(z^*))$.
\end{definition}
LASE have the desirable property that they are locally exponentially attracting under the flow of $-\omega$. This implies that a properly discretized version of $\dot z=-\omega(z)$ will also converge exponentially fast in a neighborhood of such points. LASE are the only attracting hyperbolic equilibria. Thus, making statements about all the LASE of a certain continuous-time dynamical system allows us to characterize all attracting hyperbolic equilibria.

As shown in \citet{ratliff:2013aa} and \citet{paper:Kolter}, the fact that all differential Nash equilibria are critical points of $\omega$ coupled with the structure of $J$ in zero-sum games guarantees that all differential Nash equilibria of the game are LASE of the gradient dynamics. However the converse is not true. The structure present in zero-sum games is not enough to ensure that the differential Nash equilibria are the only LASE of the gradient dynamics. When either $D^2_{xx}f$ or $D^2_{yy}f$ is indefinite at a critical point of $\omega(z)$, the Jacobian can still have eigenvalues with strictly positive real parts. 

\begin{example}
    Consider a matrix $M \in \mb{R}^{2\times 2}$ having the form:
    \begin{align*}
        M=\bmat{a & c \\ -c & -b },
    \end{align*}
    where $a,b \in \mb{R}$ and $a,b>0$. These conditions imply that $M$ cannot be the Jacobian of $\omega$ at an local Nash equilibria. However, if $b<a$ and $c^2>ab$, both of the eigenvalues of $M$ will have strictly positive real parts, and such a point could still be a LASE of the gradient dynamics.
\end{example}

Such points, which we refer to as non-Nash LASE of ~\eqref{eq:ctomega}, are what makes having guarantees on the convergence of algorithms in zero-sum games particularly difficult. Non-Nash LASE are not locally optimal for both players, and may not even be optimal for one of the players. By definition, at least one of the two players has a direction in which they would move to unilaterally decrease their cost. Such points arise solely due to the gradient dynamics, and persist even in other gradient-based dynamics suggested in the literature. In Appendix~\ref{app:counter}, we show that three recent algorithms for finding local Nash equilibria in zero-sum continuous games---consensus optimization, symplectic gradient adjustment, and a two-time scale version of simGD---are susceptible to converge to such points and therefore have no guarantees of convergence to local Nash equilibria. We note that such points can be very common since every saddle point of $f$ that is not a local Nash equilibrium is a candidate non-Nash LASE of the gradient dynamics. Further, local minima or maxima of $f$ could also be non-Nash LASE of the gradient dynamics. 

To understand how non-Nash equilibria can be attracting under the flow of $-\omega$, we again analyze the Jacobian of $\omega$. At such points, the symmetric matrix $S(z)$ must have both positive and negative eigenvalues. The sum of $S$ with $A$, however, has eigenvalues with strictly positive real part. Thus, the anti-symmetric matrix $A(z)$ can be seen as stabilizing such points. 

Previous gradient-based algorithms for zero-sum games have also pinpointed the matrix $A$ as the source of problems in zero-sum games, however they focus on a different issue. Consensus optimization \citep{paper:CO} and the symplectic gradient adjustment \citep{paper:ACO} both seek to adjust the gradient dynamics to reduce oscillatory behaviors in neighborhoods of stable equilibria. Since the matrix $A(z)$ is anti-symmetric, it has only imaginary eigenvalues. If it dominates $S$, then the eigenvalues of $J$ can have a large imaginary component. This leads to oscillations around equilibria that have been shown empirically to slow down convergence \citep{paper:CO}. Both of these adjustments rely on tunable hyper-parameters to achieve their goals. Their effectiveness is therefore highly reliant on the choice of parameter. Further, as shown in Appendix~\ref{app:counter} neither of the adjustments are able to rule out convergence to non-Nash equilibria. 

A second promising line of research into theoretically sound methods of finding the Nash equilibria of zero-sum games has approached the issue from the perspective of variational inequalities \citep{VIP,VIP2}. In \cite{VIP} extragradient methods were used to solve coherent saddle point problems and reduce oscillations when converging to saddle points. In such problems, however, all saddle points of the function $f$ are assumed to be local Nash equilibria, and thus the issue of converging to non-Nash equilibria is assumed away. Similarly, by assuming that $\omega$ is monotone, as in the theoretical treatment of the averaging scheme proposed in \cite{VIP2}, the cost function is implicitly assumed to be convex-concave. This in turn implies that the Jacobian satisfies the conditions for a Nash equilibrium everywhere. The behavior of their approaches in more general zero-sum games with less structure (like the training of GANs) is therefore not well known. Moreover, since their approach relies on averaging the gradients, they do not fundamentally change the nature of the critical points of simGD.

In the following sections we propose an algorithm for which the only LASE are the differential Nash equilibria of the game. We also show that, regardless of the choice of hyper-parameter, the Jacobian of the new dynamics at LASE has real eigenvalues, which means that the dynamics cannot exhibit oscillatory behaviors around differential Nash equilibria.

%% file: alg.tex
% !TEX root = JMLR18.tex

In this section we define the continuous-time flow that our discrete-time algorithm should ideally follow.

\begin{assumption}[Lipschitz assumptions on $f$ and $J$]
Assume that $f \in \mc{C}^3(\mathbb{R}^d,\mathbb{R})$ and $f$ and $\omega$ are $L_f$-Lipschitz and $L_\omega$-Lipschitz respectively. Finally assume that all critical points of $\omega$ are hyperbolic.
\label{ass:fJ1}
\end{assumption}

\noindent We do not require $J(z)$ to be invertible everywhere, but only at the critical points of $f$.  

Now, consider the continuous-time flow:
\begin{align}
\dot z=-h(z)=-\frac{1}{2}\left(\omega(z)+J^T(z)\left(J^T(z)J(z)+\lambda(z)I\right)^{-1}J^T(z)\omega(z)\right),
\label{eq:ctalg}
\end{align}
where $\lambda \in \mc{C}^2(\mb{R}^d,\mb{R})$ is such that $0\le\lambda(z)\le \xi$ for all $z \in \mb{R}^d$ and $\xi>0$ and $\lambda(z)=0 \iff \omega(z)=0$. 

The function $\lambda$ ensures that, even when $J$ is not invertible everywhere, the inverse matrix in \eqref{eq:ctalg} exists. The vanishing condition $\lambda(z)=0 \iff \omega(z)=0$ ensures us that the Jacobian of the adjustment term is exactly $J^T(z)$ at differential Nash equilibria. 

The dynamics introduced in \eqref{eq:ctalg} can be seen as an adjusted version of the gradient dynamics where the adjustment term only allows trajectories to approach critical points of $\omega$ along the players' axes. If a critical point is not locally optimal for one of the players (i.e., it is a non-Nash critical point) then that player can push the dynamics out of a neighborhood of that point. The mechanism is easier to see if we assume $J$ is invertible and set  $\lambda(z)\equiv0$. This results in the following dynamics:
\begin{align}
\dot z=-\frac{1}{2}\left(\omega(z)+J^T(z)J^{-1}(z)\omega(z)\right).
\label{eq:ct_goal}
\end{align}
In this simplified form we can see that the Jacobian of the adjustment is approximately $J^T$ when $||\omega(z)||_2$ is small. This approximation is exact at critical points of $\omega$. Adding this adjustment term to $\omega$ exactly cancels out the rotational part of the vector field contributed by the antisymmetric matrix $A(z)$ in a neighborhood of critical points. Since we identified $A(z)$ as the source of oscillatory behaviors and non-Nash equilibria in Section~\ref{sec:prelims}, this adjustment addresses both of these issues. The following theorem establishes this formally.

\begin{theorem}
\label{thm:algprops}
Under Assumption~\ref{ass:fJ1} and if $J^T(z)\left(J^T(z)J(z)+\lambda(z)\right)^{-1}J^T(z)\omega(z)\ne -\omega(z) \ \  \forall \ \ z\in\mb{R}^d$, the continuous-time dynamical system $\dot z=-h(z)$ satisfies:
\begin{itemize}
\item $z$ is a LASE of $\dot z=-h(z)$  $\iff$ $z$ is a differential Nash equilibrium of the game $\{ (f,-f),\mathbb{R}^d\}$.
\item If $z$  is a critical point of $\dot z=-h(z)$, then the Jacobian of $h$ at $z$ has real eigenvalues. 
\end{itemize}
\end{theorem}

\begin{proof} 
We first show that:
\[ h(z)=0 \iff \omega(z)=0.\] 
Clearly, $\omega(z)=0 \implies h(z)=0$. To show the converse, we assume that $h(z)=0$ but $\omega(z)\ne0$. This implies that:
\[J^T(z)\left(J^T(z)J(z)+\lambda(z)I\right)^{-1}J^T(z)\omega(z)=-\omega(z).\]
Since we assumed that this cannot be true, we must have that $h(z)=0 \implies \omega(z)=0$.

Having shown that under our assumptions, the critical points of $h$ are the same as those of $\omega$, we now note that the Jacobian of $h(z)$ at a critical point must have the form:
\[ J_h(z)=\frac{1}{2}\left(J(z)+J^T(z)(J^T(z)J(z))^{-1}J^T(z)J(z)\right)=\frac{1}{2}\left(J(z)+J^T(z)\right)=S(z).\]
By assumption, at critical points, $J(z)$ is invertible and $\lambda(z)=0$. Given that $\omega(z)=0$, terms that include $\omega(z)$ disappear, and the adjustment term contributes only a factor of $J^T(z)$ to the Jacobian of $h$ at a critical point. This exactly cancels out the antisymmetric part of the Jacobian of $\omega$. The Jacobian of $h$ is therefore symmetric at critical points of $\omega$ and has positive eigenvalues only when $D_{xx}^2f(z)\succ0$ and $D_{yy}^2f(z)\prec 0$. 

Since these are also the conditions for differential Nash equilibria, all differential Nash equilibria of $\mc{G}$ must be LASE of $\dot z=-h(z)$. Further, non-Nash LASE of $\dot z =-\omega(z)$ cannot be LASE of $\dot z=-h(z)$, since by definition either $D_{xx}^2f(z)$ or $D_{yy}^2f(z)$ is indefinite at such points. 
To show the second part of the theorem, we simply note that $J_h$ must be symmetric at all critical points which in turn implies that it has only real eigenvalues.
\end{proof}

The  continuous-time dynamical system therefore solves both of the problems we highlighted in Section~\ref{sec:prelims}, for any choice of the function $\lambda$ that satisfies our assumptions. The assumption that $\omega(z)$ is never an eigenvector of $J^T(z)\left(J^T(z)J(z)+\lambda(z)I\right)^{-1}J^T(z)$ with an eigenvalue of $-1$ ensures that the adjustment does not create new critical points. In high dimensions the assumption is mild since the scenario is extremely specific, but it is also possible to show that this assumption can be removed entirely by adding a time-varying term to $\eqref{eq:ctalg}$ while still retaining the theoretical guarantees.  We show this in Appendix~\ref{app:tv}.

Theorem~\ref{thm:algprops} shows that the only attracting hyperbolic equilibria of the limiting ordinary differential equation (ODE) are the differential Nash equilibria of the game. Also, since $J_h(z)$ is symmetric at critical points of $\omega$, if either $D_{xx}^2f(z)$ or $-D_{yy}^2f(z)$ has at least one negative eigenvalue then such a point would be a linearly unstable equilibrium of $\dot z=-h(z)$. Such points are linearly unstable and are therefore almost surely avoided when the algorithm is randomly initialized \citep{benaim:1995aa,sastry:1999aa}. 

Theorem~\ref{thm:algprops} also guarantees that the continuous-time dynamics do not oscillate near critical points. Oscillatory behaviors, as outlined in \cite{paper:CO}, are known to slow down convergence of the discretized version of the process. Reducing oscillations near critical points is the main goal of consensus optimization \citep{paper:CO} and the symplectic gradient adjustment \citep{paper:ACO}. However, for both algorithms, the extent to which they are able to reduce the oscillations depends on the choice of hyperparameter. The proposed dynamics achieves this for any $\lambda(z)$ that satisfies our assumptions.

We close this section by noting that one can pre-multiply the adjustment term by some function $g(z)$ such that $\omega(z)=0\implies g(z)=1$ while still retaining the theoretical properties described in Theorem~\ref{thm:algprops}. Such a function can be used to ensure that the dynamics closely track a trajectory of simGD except in neighborhoods of critical points. For example, if the matrix $J$ is ill-conditioned, such a term could be used to ensure that the adjustment does not dominate the underlying gradient dynamics. In Section~\ref{sec:num} we give an example of such a damping function.

%% file: longterm.tex
% !TEX root = JMLR18.tex

Given the limiting ODE, we could perform a straightforward Euler discretization to obtain a discrete-time update having the form:
\begin{align*}
    z_{n+1}=z_n-\gamma h(z_n).
\end{align*}
However, due to the matrix inversion, such a discrete-time update would be prohibitively expensive to implement in high-dimensional parameter spaces like those encountered when training GANs. To solve this problem, we now introduce a two-timescale approximation to the continuous-time dynamics that has the same limiting behavior, but is much faster to compute at each iteration, than the simple discretization. Since this procedure serves to exactly remove the symplectic part, $A(z)$ of Jacobian in neighborhoods of hyperbolic critical points, we refer to this two-timescale procedure as local symplectic surgery (LSS). In Appendix~\ref{app:tv} we derive the two-timescale update rule for the time-varying version of the limiting ODE and show that it also has the same properties. 

The two-timescale approximation to \eqref{eq:ctalg} is given by:
\begin{align}
\begin{split}
	\label{eq:tts}
    z_{n+1}&=z_n-{a_n}h_1(z_n,v_n)\\
    v_{n+1}&=v_n-b_nh_2(z_n,v_n),
\end{split}
\end{align}
where $h_1$ and $h_2$ are defined as:
\begin{align*}
    h_1(z,v)&=\frac{1}{2}\left(\omega(z)+ J^T(z)v\right)\\
    h_2(z,v)&=J^T(z)J(z)v-J^T(z)\omega(z)+\lambda(z)v,
\end{align*}
and the sequences of step sizes $\{a_n\}_{n=0}^\infty$,$\{b_n\}_{n=0}^\infty$ satisfy the following assumptions:

\begin{assumption}[Assumptions on the step sizes] The sequences $\{a_n\}_{n=0}^\infty$ and $\{b_n\}_{n=0}^\infty$ satisfy:
\begin{itemize}
	\item $\sum_{i=1}^\infty a_i=\infty$, and  $\sum_{i=1}^\infty b_i=\infty$;
	\item $\sum_{i=1}^\infty a^2_i<\infty$, and  $\sum_{i=1}^\infty b^2_i<\infty$;
	\item $\lim_{n \rar \infty} \frac{a_n}{b_n}=0$.
\end{itemize}
\label{ass:ss1}
\end{assumption}

We note that $h_2$ is Lipschitz continuous in $v$ uniformly in $z$ under Assumption~\ref{ass:fJ1}.  

% \begin{algorithm}[tb]
% \caption{Two Timescale Symplectic Elimination Adjustment}
% \begin{algorithmic}
%     \STATE \textbf{Input:} $a_n$, $b_n$, $\lambda$, $x_0$, $y_0$, $v_0$.
%     \STATE \textbf{Initialize:} $(x,y,v,n)\leftarrow(x_0,y_0,v_0,0)$.
%     \FOR
%         \STATE $\omega \leftarrow [D_x f(x,y),-D_y f(x,y)]^T$ .
%         \STATE $J \leftarrow J(x,y)$.
%         \STATE $v \leftarrow v-b_n\left(J^TJv-J^T\omega+\lambda(x,y)v\right)$.
%         %\STATE $\bmat{x \\ y} \leftarrow \bmat{x \\ y} -a_n\left(\omega+J^Tv\right)$.
%         \STATE $n \leftarrow n+1$.
%     \ENDFOR
%     \STATE \textbf{Output:} $(x^*,y^*) \leftarrow (x,y)$
% \end{algorithmic}
% \end{algorithm}

\begin{algorithm}[tb]
   \caption{Local Symplectic Surgery}
   \label{alg:alg1}
\begin{algorithmic}
   \STATE {\bfseries Input} Functions $f$, $\omega$, $J$, $
   \lambda$; Step sizes $a_n$, $b_n$; Initial values $(x_0,y_0,v_0)$
   \STATE \textbf{Initialize} $(x,y,v,n) \leftarrow (x_0, y_0, v_0,0)$
   \WHILE{not converged}
   		  \STATE $g_x\leftarrow D_x \left[f(x,y)+\omega^T(x,y)v\right]$ 
        \STATE $g_y\leftarrow D_y \left[ -f(x,y)+\omega^T(x,y)v \right]$ 
        \STATE $g_v \leftarrow D_v \left[ ||J(x,y)v-\omega(x,y)||_2^2+\lambda(x,y)||v||_2^2 \right]$
        \STATE $x \leftarrow x-a_n g_x$
        \STATE $y \leftarrow y-a_n g_y$
        \STATE $v \leftarrow v-b_n g_v$
        \STATE $n \leftarrow n+1$
   \ENDWHILE
   \STATE \textbf{Output} $(x^*,y^*) \leftarrow (x,y)$
\end{algorithmic}
\end{algorithm}

The $v$ process performs gradient descent on a regularized version of least squares, where the regularization is governed by $\lambda(z)$. If the $v_n$ process is on a faster time scale, the intuition is that it will first converge to $(J^T(z_n)J(z_n)+\lambda(_n)I)^{-1}\omega(z_n)$, and then $z_n$ will track the limiting ODE in \eqref{eq:ctalg}. In the next section we show that this behavior holds even in the presence of noise. 

The key benefit to the two-timescale process is that $z_{n+1}$ and $v_{n+1}$ can be computed efficiently since neither require a matrix inversion. In fact, as we show in Appendix~\ref{app:jac_vec_prods}, the computation can be done with Jacobian-vector products with the same order of complexity as that of simGD, consensus optimization, and the symplectic gradient adjustment. This insight gives rise to the procedure outlined in Algorithm~\ref{alg:alg1}. 

\subsection{Long-term behavior of the two-timescale approximation}

We now show that LSS asymptotically tracks the limiting ODE even in the presence of noise. This implies that the algorithm has the same limiting behavior as \eqref{eq:ctalg}.
%We then relax many of the assumptions to show that the LASE will still be attracting with high probability, and that the high probability rate of convergence of the two-timescale process is similar to that of single timescale algorithms like simGD. This justifies the use of the proposed algorithms in larger parameter spaces.
In particular, our setup allows us to treat the case where one only has access to unbiased estimates of $h_1$ and $h_2$ at each iteration. This is the setting most likely to be encountered in practice, for example in the case of training GANs in a mini-batch setting. 

We assume that we have access to estimators $\hat h_1$ and $\hat h_2$ such that:
\begin{align*}
	\mb{E}\left[\hat h_1(z,v)\right]&=\omega(z)+J^T(z)v\\
	\mb{E}\left[\hat h_2(z,v)\right]&=J^T(z)J(z)v+J^T(z)\omega(z).
\end{align*}
To place this in the form of classical two-timescale stochastic approximation processes, we write each estimator $\hat h_1$ and $\hat h_2$ as the sum of its mean and zero-mean noise processes $M^z$ and $M^v$ respectively. This results in the following two timescale process:
\begin{align}
\label{eq:TTSsys}
\begin{split}
z_{n+1}&=z_n-a_n[\omega(z_n)+J^T(z_n)v_n +M^z_{n+1}]\\
v_{n+1}&=v_n-b_n[J^T(z_n)J(z_n)v_n+J^T(z_n)\omega(z_n) +\lambda(z_n)v_n+M^v_{n+1}].
\end{split}
\end{align}
We assume that the noise processes satisfy the following standard conditions \citep{benaim:1999ab,borkar:book}:
\begin{assumption}{Assumptions on the noise:}
Define the filtration $\mc{F}_n$: 
\[ \mc{F}_n=\sigma(z_0,v_0,M^v_1,M^z_1,...,M^z_n,M^v_n),\]
for $n\ge 0$. Given $\mc{F}_n$, we assume that:
\begin{itemize}
	\item $M_{n+1}^v$ and $M_{n+1}^z$ are conditionally independent given $\mc{F}_{n}$ for $n\ge0$.
	\item $\mb{E}[M_{n+1}^v|\mc{F}_n]=0$ and $\mb{E}[M_{n+1}^z|\mc{F}_n]=0$ for $n\ge0$.
	%\item $Pr(M_{n+1}^z>u|\mc{F}_n)\le c_1 e^{-c_2u}$ and $Pr(M_{n+1}^v>u|\mc{F}_n)\le c_3 e^{-c_4u}$ for $c_1,c_2,c_3,c_4>0$ and $u\ge \bar u>0$ for $\bar u$ sufficiently large.
	\item $\mb{E}[||M_{n+1}^z|||\mc{F}_n]\le c_z(1+||z_n||)$ and $\mb{E}[||M_{n+1}^v|||\mc{F}_n]\le c_v(1+||z_n||)$ almost surely for some positive constants $c_z$ and $c_v$.
\end{itemize}
\label{ass:noise}
\end{assumption}
Given our assumptions on the estimator, cost function, and step sizes we now show that \eqref{eq:TTSsys} asymptotically tracks a trajectory of the continuous-time dynamics almost surely. Since $h_1$, $h_2$, and $v^*(z)=(J^T(z)J(z)+\lambda(z)I)^{-1}\omega(z)$ are not uniformly Lipschitz continuous in both $z$ and $v$, we cannot directly invoke results from the literature. Instead, we adapt the proof of Theorem 2 in Chapter 6 of \cite{borkar:book} to show that $v_n \rar v^*(z_n)$ almost surely. We then invoke Proposition 4.1 from \cite{benaim:1999ab} to show that $z_n$ asymptotically tracks $\dot z=-h(z)$. We note that this approach only holds on the event $\{\sup_n ||z_n||+||v_n|| <\infty\}$. Thus, if the stochastic approximation process remains bounded, then under our assumptions we are sure to track a trajectory of the limiting ODE.

\begin{lemma}
\label{lemma:vconv}
Under Assumptions \ref{ass:fJ1}-\ref{ass:noise}, and on the event $\{\sup_n ||z_n||_2+||v_n||_2 <\infty\}$: 

\[ (z_n,v_n)\rar \{(z,v^*(z)): z \in \mathbb{R}^d\},\] 
\noindent almost surely.
\end{lemma}

\begin{proof}
We first rewrite \eqref{eq:TTSsys} as:
\begin{align*}
z_{n+1}&=z_n-b_n\left[\frac{a_n}{b_n}h_1(z_n,v_n)+\bar M^z_{n+1}\right]\\
v_{n+1}&=v_n-b_n[h_2(z_n,v_n)+M^v_{n+1}],
\end{align*}
where $\bar M^z_{n+1}=\frac{a_n}{b_n}M^z_{n+1} $. By assumption, $\frac{a_n}{b_n}\rar 0$. Since $h_1$ is locally Lipschitz continuous, it is bounded on the event $\{\sup_n ||z_n||_2+||v_n||_2 <\infty\}$. Thus, $\frac{a_n}{b_n}h_1(z_n,v_n)\rar 0$ almost surely.

From Lemma 1 in Chapter 6 of \cite{borkar:book} , the above processes, on the event $\{\sup_n ||z_n||_2+||v_n||_2 <\infty\}$, converge almost surely to internally chain-transitive invariant sets of $\dot v=-h_2(z,v)$ and $\dot z=0$. Since, for a fixed $z$, $h_2(z,v)$ is a Lipschitz continuous function of $v$ with a globally asymptotically stable equilibrium at $(J^T(z)J(z)+\lambda(z)I)^{-1}\omega(z)$, the claim follows.
\end{proof}

Having shown that $||v_n-v^*(z_n)||_2 \rar 0$ almost surely, we now show that $z_n$ will asymptotically track a trajectory of the limiting ODE. Let us first define $z(t,s,z_s)$ for $t\ge s$ to be the trajectory of $\dot z=-h(z)$ starting at $z_s$ at time $s$.

\begin{theorem}
\label{thm:tts_asym}
Given Assumptions \ref{ass:fJ1}-\ref{ass:noise}, let $t_n=\sum_{i=0}^{n-1} a_i$. On the event $\{sup_n ||z_n||_2+||v_n||_2 <\infty\}$, for any integer $K>0$ we have:
\[ \lim_{n\rar \infty} \sup_{0\le h \le K} \ \|z_{n+h}-z(t_{n+h},t_n, z_n)\|_2=0.\]
\end{theorem}

\begin{proof}
The proof makes use of Propositions 4.1 and 4.2  in \cite{benaim:1999ab} which are supplied in Appendix~\ref{app:props}. 

We first rewrite the $z_n$ process as:
\[ z_{n+1}=z_n-a_n\left[ h(z)-J^T(z_n)\left(v^*(z_n)-v_n\right)+ M^z_{n+1}\right].\]
We note that, from Lemma~\ref{lemma:vconv}, $\left(v^*(z_n)-v_n\right) \rar 0$ almost surely.  Since $||J^T(z_n)||_2 < L_{\omega}$, we can write this process as:
\[ z_{n+1}=z_n-a_n\left[ h(z)-\chi_n+ M^z_{n+1}\right],\]
where $\chi_n \rar 0$ almost surely. Since $h$ is continuously differentiable, it is locally Lipschitz, and on the event $\{sup_n ||z_n||+||v_n|| <\infty\}$ it is bounded. It thus induces a continuous globally integrable  vector field, and therefore satisfies the assumptions for Propositions 4.1 in \cite{benaim:1999ab}. Further, by assumption the sequence of step sizes and martingale difference sequences satisfy the assumptions of Proposition 4.2 in \cite{benaim:1999ab}. Invoking Proposition 4.1 and 4.2 in \cite{benaim:1999ab} gives us the desired result. 
\end{proof}

Theorem~\ref{thm:tts_asym} guarantees that LSS asymptotically tracks a trajectory of the limiting ODE. The approximation will therefore avoid non-Nash equilibria of the gradient dynamics. Further, the only locally asymptotically stable points for LSS must be the differential Nash equilibria of the game.

%% file: num_ex.tex
% !TEX root = JMLR18.tex
We now present two numerical examples that illustrate the performance of both the limiting ODE and LSS. The first is a zero-sum game played over a function in $\mb{R}^2$ that allows us to observe the behavior of both the limiting ODE around both local Nash and non-Nash equilibria. In the second example we use LSS to train a small generative adversarial network (GAN) to learn a mixture of eight Gaussians. Further numerical experiments and comments are provided in Appendix~\ref{app:numex}.

\subsection{2-D example}
For the first example, we consider the game based on the following function $f$ in $\mb{R}^2$ :
\[ f(x,y)= e^{-0.01(x^2+y^2)}((0.3x^2+y)^2+(0.5y^2+x)^2).\]
This function is a fourth-order polynomial that is scaled by an exponential to ensure that it is bounded. The gradient dynamics $\dot z=-\omega(z)$ associated with function have four LASE. By evaluating the Jacobian of $\omega$ at these points we find that three of the LASE are local Nash equilibria. These are denoted by `x' in Figure~\ref{fig:toy1}. The fourth LASE is a non-Nash equilibrium which is denoted with a star. In Figure~\ref{fig:toy1}, we plot the sample paths of both simGD and our limiting ODE from the same initial positions, shown with red dots. We clearly see that simGD converges to all four LASE, depending on the initialization. Our algorithm, on the other hand, only converges to the local Nash equilibria. When initialized close to the non-Nash equilibrium it diverges from the simGD path and ends up converging to a LNE. 
\begin{figure}[h]
 % {fig:subfigex}
\center    
      \includegraphics[width=0.75\columnwidth]{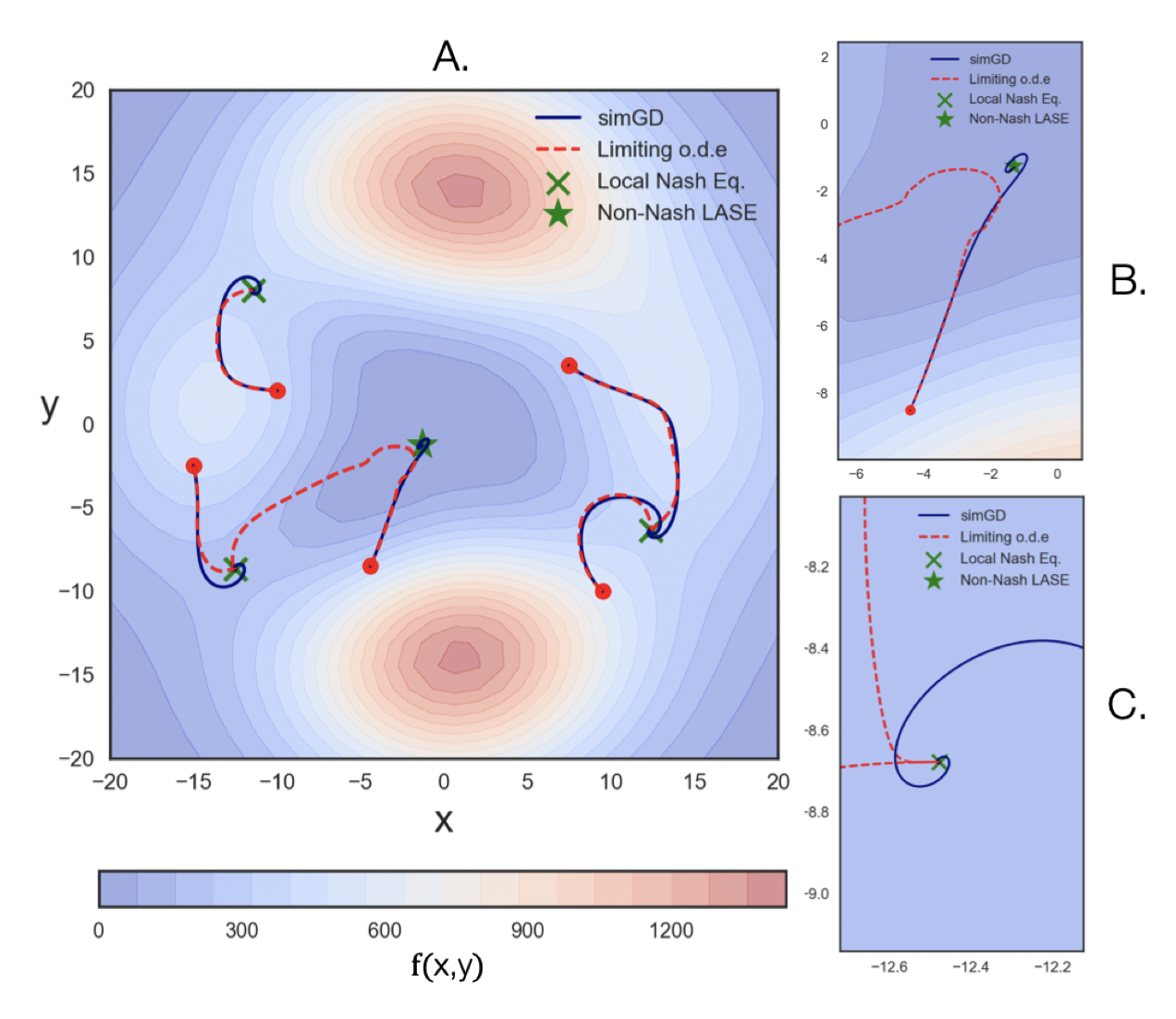}
        \caption{A. Convergence of simGD and the limiting ODE \eqref{eq:ctalg} to different LASE of the game. The local Nash equilibria are denoted by 'x' and the non-Nash equilibrium by a '*'. The heat map is used to show that this game is being played over a function and the Nash equilibria and non-Nash equilibria are all at different payoffs. B. From the same initialization where simGD converges to the non-Nash equilibrium, the ODE converges instead to a local Nash equilibria. C. We empirically validate that the ODE does not exhibit oscillatory behaviors near local Nash equilibria.}
  \label{fig:toy1}
\end{figure}
This numerical example also allows us to study the behavior of our algorithm around LASE. By focusing on a local Nash equilibrium, as in Figure~\ref{fig:toy1}B, we observe that the limiting ODE approaches it directly even when simGD displays oscillatory behaviors. This empirically validates the second part of Theorem~\ref{thm:algprops}.

In Figure~\ref{fig:toy2} we empirically validate that LSS asymptotically tracks the limiting ODE. When the fast timescale has not converged, the process tracks the gradient dynamics. Once it has converged however, we see that it closely tracks the limiting ODE which leads it to converge to only the local Nash equilibria. This behavior highlights an issue with the two-timescale approach. Since the non-Nash equilibria of the gradient dynamics are saddle points for the new dynamics they can slow down convergence. However, the process will eventually escape such points  \citep{benaim:1999ab}. 

\begin{figure}[h]
 % {fig:subfigex}
\center    
      \includegraphics[width=0.75\columnwidth]{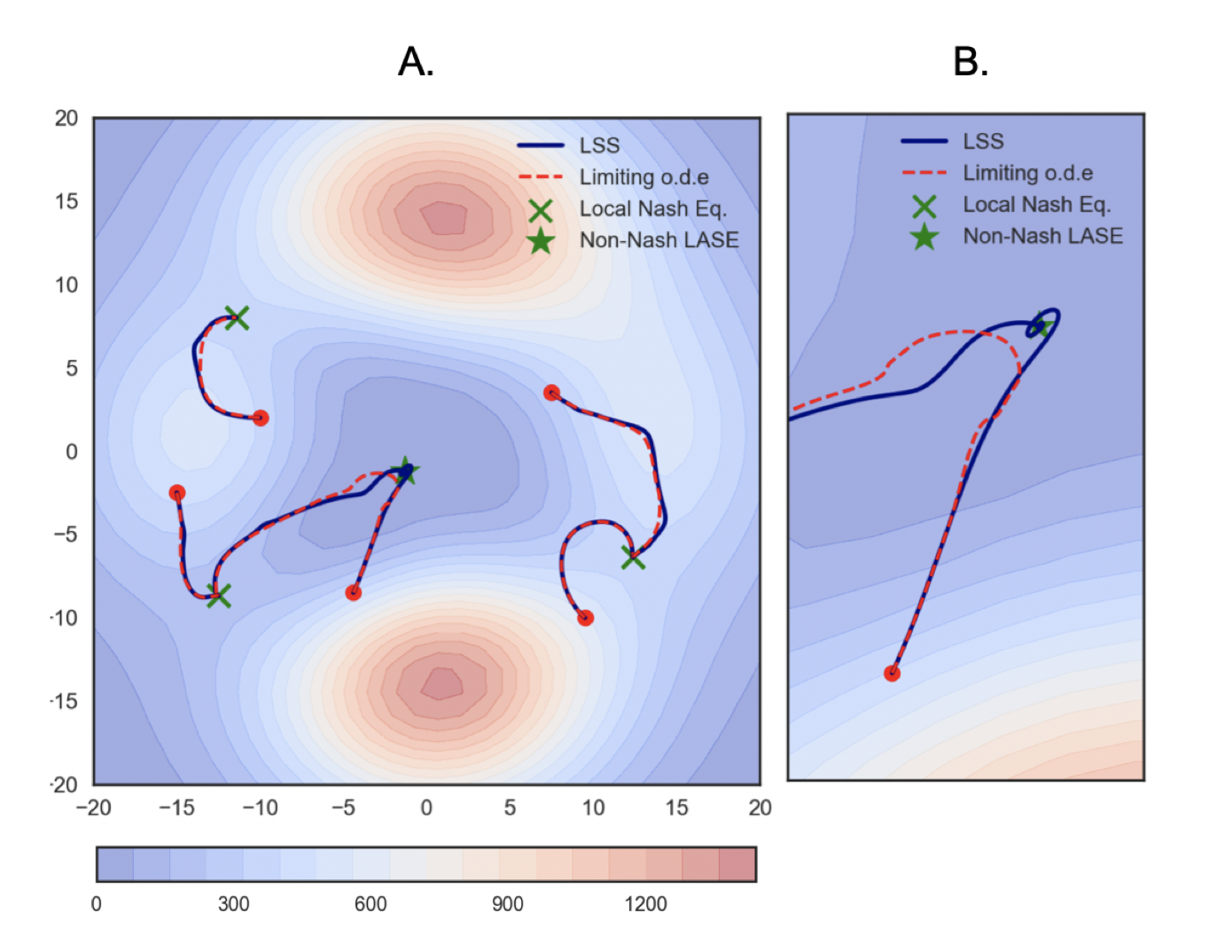}
        \caption{Behavior of the two-timescale approximation with respect to the limiting ODE. A. The two-timescale procedure closely tracks the limiting ODE from all four initializations. B. Focusing on the non-Nash equilibrium, the two-timescale procedure tracks the gradient-dynamics closely which leads it to begin oscillating around the equilibrium. Once the fast timescale has converged however, the process quickly exits the neighborhood of the non-Nash equilibrium and closely tracks the limiting ODE.}
  \label{fig:toy2}
\end{figure}

In our numerical experiments we let $\lambda(z)=\xi_1\left(1-e^{-||\omega(z)||^2} \right)$. We also make use of a damping function $g$ as described in Section~\ref{sec:alg}. The limiting ODE is therefore given by:
\[ \dot z=-(\omega(z)+e^{-\xi_2 ||v||^2}v),\]
where $v=J^T(z)(J^T(z)J(z)+\lambda(z)I)^{-1}J^T(z)\omega(z)$. For the two-timescale process, since there is no noise we use constant step sizes and the following update:
\begin{align*}
z_{n+1}&=z_n-\gamma_1(\omega(z_n)+e^{-\xi_2 ||J^T(z_n)v_n||^2}J^T(z_n)v_n) \\
v_{n+1}&=v_n-\gamma_2(J^T(z_n)J(z_ n)v_n+\lambda(z_n)v_n-J^T(z_n)\omega(z_n)),
\end{align*}
where $\xi_1=1e-4$,$\xi_2=1e-4$,$\gamma_1=0.004$, and $\gamma_2=0.005$.

\subsection{Generative adversarial network}

We now train a generative adversarial network with LSS. Both the discriminator and generator are fully connected neural networks with four hidden layers of 16 neurons each. The tanh activation function is used since it satisfies the smoothness assumptions for our functions. For the latent space, we use a 16-dimensional Gaussian with mean zero and covariance $\Sigma=0.1 I_{16}$. The ground truth distribution is a mixture of eight Gaussians with their modes uniformly spaced around the unit circle and covariance $\Sigma=1e-4 I_2$.
\begin{figure}[h]
 % {fig:subfigex}
\center    
      \includegraphics[width=\columnwidth]{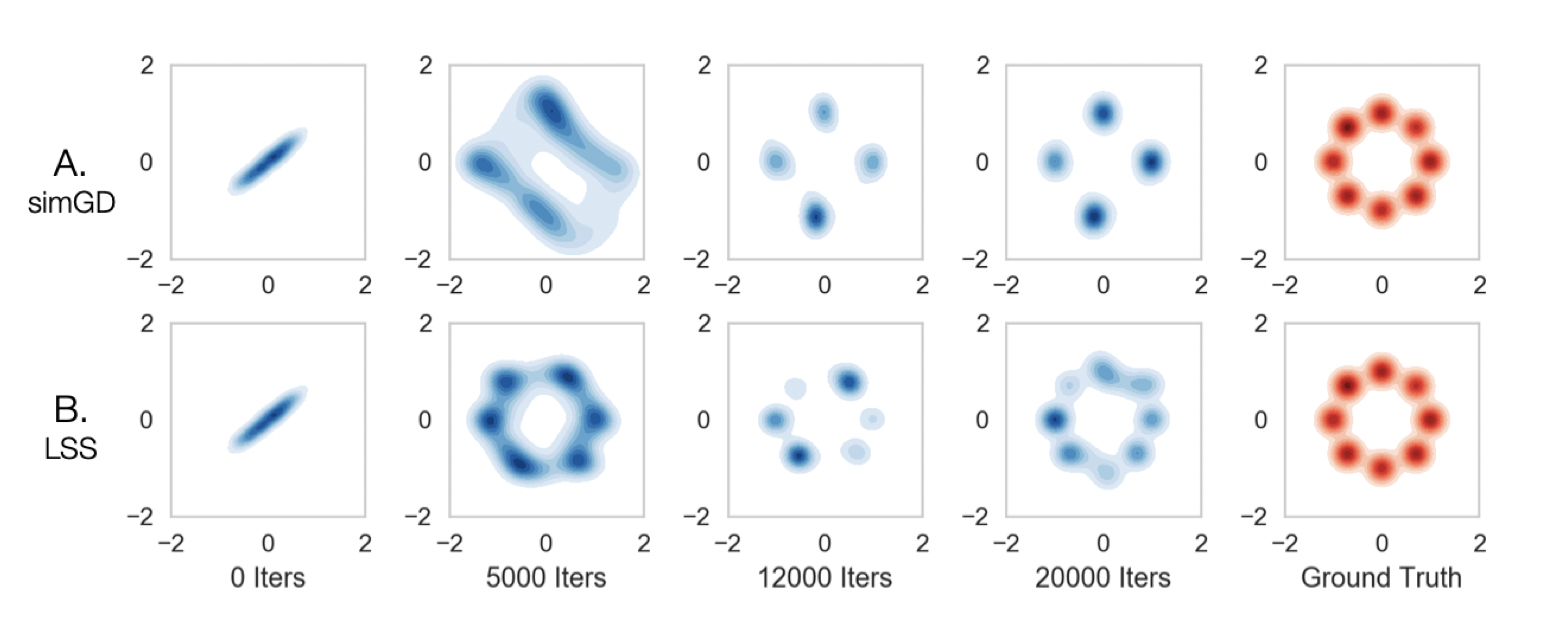}
        \caption{A GAN trained with A. simGD and B. LSS.}
  \label{fig:gan}
\end{figure}
In Figure~\ref{fig:gan}, we show the progression of the generator at $0$, $5000$, $12000$, and $20000$ iterations for a GAN initialized with the same weights and biases and then trained with A. simGD and B. LSS. We can see empirically that, in this example, LSS converges to the true distribution while simGD quickly suffers mode collapse, showing how the adjusted dynamics can lead to convergence to better equilibria. Further numerical experiments are shown in Appendix~\ref{app:numex}.

We caution that convergence rate per se is not necessarily a reasonable metric on which to compare performance in the GAN setting or in other game-theoretic settings.  Competing algorithms may converge faster than our method when used to train GANs, but once because the competitors could be converging quickly to a non-Nash equilibrium, which is not desirable. Indeed, the optimal solution is known to be a local Nash equilibrium for GANs \citep{goodfellow:2014ab,paper:Kolter}. LSS may initially move towards a non-Nash equilibrium, while subsequently escaping the neighborhood of such points before converging. This will lead to a slower convergence rate, but a better quality solution.

%% file: discussion.tex
% !TEX root = ICML2019.tex

We have introduced local symplectic surgery, a new two-timescale algorithm for finding the local Nash equilibria of two-player zero-sum continuous games. We have established that this comes with the guarantee that the only hyperbolic critical points to which it can converge are the local Nash equilibria of the underlying game. This significantly improves upon previous methods for finding such points which, as shown in Appendix~\ref{app:counter}, cannot give such guarantees. We have analyzed the asymptotic properties of the proposed algorithm and have shown that the algorithm can be implemented efficiently. Altogether, these results show that the proposed algorithm yields limit points with game-theoretic relevance while ruling out oscillations near those equilibria and having a similar per-iteration complexity as existing methods which do not come with the same guarantees. Our numerical examples allow us to empirically observe these properties.

It is important to emphasize that our analysis has been limited to neighborhoods of equilibria; the proposed algorithm can  converge in principle to limit cycles at other locations of the space.  These are hard to rule out completely. Moreover, some of these limit cycles may actually have some game-theoretic relevance \citep{hommes:2012aa,benaim:1997ab}. Another limitation of our analysis is that we have assumed the existence of local Nash equilibria in games. Showing that they exist and finding them is very hard to do in general. Our algorithm will converge to local Nash equilibria, but may diverge when the game does not admit equilibria or when the algorithm does not come any equilibria its region of attraction. Thus, divergence of our algorithm is not a certificate that no equilibria exist. Such caveats, however, are the same as those for other gradient-based approaches for finding local Nash equilibria. 

Another drawback to our approach is the use of second-order information. Though the two-timescale approximation does not need access to the full Jacobian of the gradient dynamics, the update does involve computing Jacobian-vector products. This is similar to other recently proposed approaches but will be inherently slower to compute than pure first- or zeroth-order methods. Bridging this gap while retaining similar theoretical properties remains an interesting avenue of further research.

In all, we have shown that some of the inherent flaws to gradient-based methods in zero-sum games can be overcome by designing our algorithms to take advantage of the game-theoretic setting. Indeed, by using the structure of local Nash equilibria we designed an algorithm that has significantly stronger theoretical support than existing approaches.

%% file: time_varying.tex
% !TEX root = JMLR18.tex

In this section we analyze a slightly different version of \eqref{eq:ctalg} that allows us to remove the assumption that $\omega(z)$ is never an eigenvector of $J^T(z)(J^T(z)J(z)+\lambda(z)I)^{-1}J^T(z)$ with associated eigenvalue $-1$. Though this assumption is relatively mild, since intuitively it will be very rare that $\omega$ is exactly the eigenvector of the adjustment matrix, we show that by adding a third term to \eqref{eq:ctalg} we can remove it entirely while retaining our theoretical guarantees. The new dynamics are constructed by adding a time-varying term to the dynamics that goes to zero only when $\omega(z)$ is zero. This guarantees that the only critical points of the limiting dynamics are the critical points of $\omega$. The analysis of these dynamics is slightly more involved and requires generalizations of the definition of a LASE to handle time-varying dynamics. We first define an equilibrium of a potentially time-varying dynamical system $\dot\theta=g(\theta,t)$ as a point $\theta^*$ such that $g(\theta^*,t)\equiv0$ for all $t\ge0$. We can now generalize the definition of a LASE to the time-varying setting.

\begin{definition}
A strategy $\theta^*\in \mb{R}^d$ is a \emph{locally uniformly asymptotically stable equilibrium} of the time-varying continuous time dynamics
$\dot \theta=-f(\theta,t)$ if $\theta^*$ is an equilibrium of $\dot \theta=-f(\theta,t)$, and $D_\theta f(\theta^*,t)\equiv J(\theta^*)$ and $Re(\lambda)>0$ for all $\lambda\in\spec(J(z^*))$.    
\end{definition}

Locally uniformly asymptotically stable equilibria, under this definition, also have the property that they are locally exponentially attracting under the flow, $\dot \theta=-f(\theta,t)$. Further, since the linearization around a locally uniformly asymptotically stable equilibrium is time-invariant, we can still invoke converse Lyapunov theorems like those presented in \cite{sastry:1999aa} when deriving the non-asymptotic bounds. 

Having defined equilibria and a generalization of LASE for time-varying systems, we now introduce a time-varying version of the continuous-time ODE presented in Section~\ref{sec:alg} which allows us to remove the assumption that $\omega(z)$ is never an eigenvector of $J^T(z)(J^T(z)J(z)+\lambda(z)I)^{-1}J^T(z)$ with associated eigenvalue $-1$. The limiting ODE is given by:
\begin{align}
\dot z=-h_{TV}(z,t)=-(h(z)+g_{TV}(z,t)),
\label{eq:TV_alg}
\end{align}
where $h(z)$ is as described in Section~\ref{sec:alg}, $g(z,t)$ can be decomposed as:
\[ g_{TV}(z,t)=\lambda_1(z)u(t),\]
where $\lambda_1 \in \mc{C}^2(\mb{R}^d,\mb{R})$ satisfies:
\begin{itemize}
\item $0\le\lambda_1(z)\le \xi_2$ for all $z \in \mb{R}^d$.
\item $\lambda_1(z)=0 \iff \omega(z)=0$.
\item $\omega(z)=0 \implies D_z\lambda_1(z)=0$,
\end{itemize}
and where $u:\mb{R}_+\rar\mb{R}^d$ satisfies:
\begin{itemize}
\item $\nexists \ \ t_0\ge0$ such that $D_t u(t)\equiv 0 \ \ \forall \ \ t\ge t_0$.
\item $||u(t)||\le \xi_3 \ \ \forall \ \ t\ge0$.
\end{itemize}

Thus we require that the time-varying adjustment term $g_{TV}$ must be bounded and is equal to zero only when $\omega(z)=0$. Most importantly, we require that for any $z$ that is not a critical point of $\omega$, $g_{TV}$ must be changing in time. An example of a $g_{TV}$ that satisfies these requirements is:
\begin{align}
g_{TV}(z,t)=\xi_1\left(1-e^{-\xi_2||\omega(z)||^2} \right)\cos(t)u_0,
\label{eq:tvadj}
\end{align}
for $u_0 \in \mb{R}^d$, $u_0\ne 0$ and $\xi_1,\xi_2>0$.

These conditions, as the next theorem shows, allow us to guarantee that the only locally asymptotically stable equilibria are the differential Nash equilibria of the game.

\begin{theorem}
\label{thm:tv}
Under Assumption~\ref{ass:fJ1} the continuous-time dynamical system $\dot z=-h_{TV}(z,t)$ satisfies:
\begin{itemize}
\item $z$ is a locally uniformly asymptotically stable equilibrium of $\dot z=-h_{TV}(z,t)$  $\iff$ $z$ is a DNE of the game $\{ (f,-f),\mathbb{R}^d\}$.
\item If $z$  is an equilibrium point of $\dot z=-h_{TV}(z,t)$, then the Jacobian of $h_{TV}$ at $z$ is time-invariant and has real eigenvalues. 
\end{itemize}
\end{theorem}

\begin{proof} 
We first show that:
\[ h_{TV}(z,t)\equiv 0 \ \ \forall t\ge0 \iff \omega(z)=0.\] 
By construction $\omega(z)=0 \implies h_{TV}(z,t)\equiv 0 \ \ \forall t\ge0$. 
To show the converse, we assume that there exists a $z$ such that $h_{TV}(z,t)\equiv0 \ \ \forall \ \ t\ge0$ but $\omega(z)\ne0$. This implies that:
\begin{align*}
-g_{TV}(z,t)&=\omega(z)+ J^T(z)\left(J^T(z)J(z)+\lambda(z)I\right)^{-1}J^T(z)\omega(z) \quad \forall t\ge0.
\end{align*}
Since $z$ is a constant and $ \lambda_1(z)>0$, taking the derivative of both sides with respect to $t$ gives us the following condition on $g$ under our assumption:
\[ D_tu(t)=0 \ \ \forall t\ge0.\]
By assumption this cannot be true. Thus, we have a contradiction and $h_{TV}(z,t)\equiv 0 \ \ \forall t\ge0 \implies \omega(z)=0$. 

Having shown that the critical points of $\dot z=-h_{TV}(z,t)$ are the same as that of $\dot z=-\omega(z)$, we now note that the Jacobian of $h_{TV}(z,t)$, at critical points, must be $S(z)$. Under the same development as the proof of Theorem~\ref{thm:algprops} the Jacobian of $h_{TV}$ is given by:
\[ J_{TV}(z)=J(z)+J^T(z)+(D_zg_{TV}(z,t)).\]

Again, by construction $D_zg_{TV}(z,t)=0$ when $\omega(z)=0$. The third term therefore disappears and we have that $J_{TV}(z)=S(z)$. The proof now follows from that of Theorem~\ref{thm:algprops}.
\end{proof}

We have shown that adding a time-varying term to the original adjusted dynamics allows us to remove the assumption that the adjustment term is never exactly $-\omega(z)$. As in Section~\ref{sec:alg} we can now construct a two-timescale process that asymptotically tracks \eqref{eq:TV_alg}. We assume that $u(t)$ is a deterministic function of a trajectory of an ODE: 
\[\dot \theta=-h_3(\theta),\]
with a fixed initial condition $\theta_0$ such that $u(t)=g_{TV}(\theta(t))$. We assume that $f$ is Lipschitz-continuous and $g_{TV}$ is continuous and bounded. Note that under our assumptions, $||g_{TV}(\theta)||\le \xi_3$ for all $\theta \in \mb{R}^{d_\theta}$.

The form of $u(t)$ introduced in \eqref{eq:tvadj}, $u(t)=\cos(t)u_0$ can be written as $u(t)=([1,0]*\theta(t))u_0$, where $\theta$ satisfies the linear dynamical system:
\[\dot \theta=\bmat{0 & -1 \\ 1 & 0}\theta, \]
with $\theta_0=[1,0]$. 

Given this setup, the continuous-time dynamics can be written as:
\begin{align}
\begin{split}
\label{eq:TV_alg2}
\dot \theta&=-h_3(\theta)\\
\dot z&=-h_4(z,\theta),
\end{split}
\end{align}
where:
\begin{align*}
h_4(z,\theta)&=\frac{1}{2}(\omega(z)+J^T(z)\left(J^T(z)J(z)+\lambda(z)I\right)^{-1}J^T(z)\omega(z)+\lambda_1(z)u(\theta)). 
\end{align*}

 Having made this further assumption on the time-varying term, we now introduce the two-timescale process that asymptotically tracks \eqref{eq:TV_alg2}. This process is given by:
\begin{align}
\begin{split}
	\label{eq:tts_tv}
	\theta_{n+1}&=\theta_n-a_nh_3(\theta_n)\\
    z_{n+1}&=z_n-{a_n}\hat h_5(z_n,v_n,\theta_n)\\
    v_{n+1}&=v_n-b_n\hat h_6(z_n,v_n),
\end{split}
\end{align}
where $\theta_0=\bar \theta$
\begin{align*}
    \mb{E}[\hat h_5(z,v,\theta)]&=h_5(z,v,\theta):=\frac{1}{2}\left(\omega(z)+ J^T(z)v\right)+\lambda_1(z)u(\theta)\\
    \mb{E}[\hat h_6(z,v)]&=h_6(z,v):=J^T(z)J(z)v-J^T(z)\omega(z)+\lambda(z)v.
\end{align*}

Proceeding as in Section~\ref{sec:alg}, we write $\hat h_5(z,v,\theta)=h_5(z,v,\theta)+M^z$ and $\hat h_6(z,v)=h_6(z,v)+M^v$ where $M^z$ and $M^v$ are martingale difference sequences satisfying Assumption~\ref{ass:noise}. We note that the $\theta_n$ process is deterministic.

This two-timescale process gives rise to the time-varying version of local symplectic surgery (TVLSS) outlined in Algorithm~\ref{alg:alg2}.
\begin{algorithm}[tb]
   \caption{Time-varying Local Symplectic Surgery}
   \label{alg:alg2}
\begin{algorithmic}
   \STATE {\bfseries Input} Functions $f$, $\omega$, $J$, $
   \lambda$, $g_{TV}$; Step sizes $a_n$, $b_n$; Initial values $(x_0,y_0,v_0)$
   \STATE \textbf{Initialize} $(x,y,v,n) \leftarrow (x_0, y_0, v_0,0)$
   \WHILE{not converged}
   		  \STATE $g_x\leftarrow D_x \left[f(x,y)+\omega^T(x,y)v\right]$ 
        \STATE $g_y\leftarrow D_y \left[ -f(x,y)+\omega^T(x,y)v \right]$ 
        \STATE $g_v \leftarrow D_v \left[ ||J(x,y)v-\omega(x,y)||_2^2+\lambda(x,y)||v||_2^2 \right]$
        \STATE $x \leftarrow x-a_n (g_x+g_{TV}(x,y,n))$
        \STATE $y \leftarrow y-a_n (g_y+g_{TV}(x,y,n))$
        \STATE $v \leftarrow v-b_n g_v$
        \STATE $n \leftarrow n+1$
   \ENDWHILE
   \STATE \textbf{Output} $(x^*,y^*) \leftarrow (x,y)$
\end{algorithmic}
\end{algorithm}

 We can now proceed as in Section~\ref{sec:thms} to show that \eqref{eq:tts_tv} has the same limiting behavior as the limiting ODE in \eqref{eq:TV_alg2}. 

\begin{lemma}
\label{lemma:vconv_tv}
Under Assumptions \ref{ass:fJ1}-\ref{ass:noise}, and on the event $\{\sup_n ||z_n||+||v_n||+||\theta_n|| <\infty\}$: 
\[ (\theta_n,z_n,v_n)\rar \{(\theta,z,v^*(z)): (z, \theta) \in \mb{R}^d \times \mb{R}^{d_\theta}\}\] 
almost surely.
\end{lemma}

The proof of this lemma is exactly the same as that of Lemma~\ref{lemma:vconv}, and we do not repeat the argument. We can now invoke Proposition 4.1 and 4.2 from \cite{benaim:1999ab} to show that the two-timescale process tracks a trajectory of the limiting ODE. For $t \ge s$, let $\theta(t,s,\theta_s)$ and $z(t,s,z_s)$ be the trajectories of the dynamical systems introduced in \eqref{eq:TV_alg2}, starting at states $\theta_s$ and $z_s$ respectively at time $s$.

\begin{theorem}
\label{thm:tts_asym_tv}
Let Assumptions \ref{ass:fJ1}-\ref{ass:noise} hold.  Further, assume $h_3$ is Lipschitz continuous and define $t_n=\sum_{i=0}^{n-1} a_i$. Then on the event $\{sup_n ||z_n||+||v_n||+|\theta_n|| <\infty\}$, for any integer $K>0$:
\[ \lim_{n\rar \infty} \sup_{0\le h \le K} \ ||z_{n+h}-z(t_{n+h},t_n, z_n)\|=0\]

\[ \lim_{n\rar \infty} \sup_{0\le h \le K} \ ||\theta_{n+h}-\theta(t_{n+h},t_n, z_n)\|=0.\]
\end{theorem}
The proof of Theorem~\ref{thm:tts_asym_tv} is exactly like that of Theorem~\ref{thm:tts_asym}. 

Together, Theorem~\ref{thm:tv} and Theorem~\ref{thm:tts_asym_tv} show that our approach can be modified to allow us to relax some of our assumptions, while still maintaining the theoretical properties of our algorithms.

%% file: counter_example.tex
% !TEX root = JMLR18.tex

We now show that three state-of-the art algorithms for finding the local Nash equilibria of zero-sum games are all attracted to non-Nash equilibria. This implies that these algorithms do not have guarantees on the local optimality of their limit points and cannot give guarantees on convergence to local Nash equilibria. To do this we construct a simple counter-example in $\mb{R}^2$ that demonstrates that all of the algorithms are attracted to some non-Nash limit point. The particular game we use is described below.

\begin{example}
\label{ex:counterex}
	Consider the game $\mc{G}=\{(f,-f),\mb{R}^2\}$ where:
		\[ f(x,y)=\frac{1}{2}\bmat{x \\ y}^T\bmat{1 & 1 \\ 1 & 0.1}\bmat{x \\ y}.\]
	Staying with our earlier notation, the player with variable $x$ seeks to minimize $f$, and the player with variable $y$ minimizes $-f$.

	This game has only one critical point, $(x,y)=(0,0)$, and the combined gradient dynamics for this game are linear and are given by:
	\[ \omega(x,y)=\bmat{1 & 1 \\- 1 & -0.1}\bmat{x \\ y}.\]
	For all $x,y \in \mb{R}$ the Jacobian of $\omega$ is given by:
	\[ J(x,y)=\bmat{1& 1 \\- 1 & -0.1}.\]
	Since the diagonal is not strictly positive, $(x,y)=(0,0)$ is not an LNE. However, since the eigenvalues of $J(x,y)$ are $0.45+0.835i$ and $0.45-0.835i$, $(x,y)=(0,0)$ is a LASE of the gradient dynamics.
\end{example}

We first show that even running SGA on two timescales as suggested by \cite{paper:wrongTTS}, cannot guarantee convergence to only local Nash equilibria. We assume that the maximizing player is on the slower timescale, and, in the following proposition, show that SGA on two timescales can still converge to non-Nash fixed points.

\begin{proposition}
	Simultaneous gradient ascent on two timescales can converge to non-Nash equilibria.
\end{proposition}

\begin{proof}
	Consider the game introduced in Example~\ref{ex:counterex}, and the following dynamics:
	\begin{align*}
		x_{n+1}=x_n -a_n (\omega_1(x_n,y_n))\\
		y_{n+1}=y_n -b_n (\omega_2(x_n,y_n)),
	\end{align*}
	where $\omega_i$ denotes the $ith$ component of $\omega$ as described in Example~\ref{ex:counterex} and $a_n,b_n$ are sequences of step sizes satisfying Assumption~\ref{ass:ss1}.

	Since the dynamics $\dot x=-x-y$ have a unique equilibrium $x=-y$ for a fixed $y \in \mb{R}$ and the $x$ process is on the faster timescale, Chapter 4 in \cite{borkar:book}, assures us that $x_n\rar -y_n$ asymptotically. Now, assuming that $x_n$ has converged, we analyze the slower timescale. Plugging in for $x_n$, we get that, asymptotically, the dynamics will track:
	\[ \dot y=-y+0.1y=-0.9y.\]
	Since these dynamics have a unique (exponentially) stable equilibrium at $y=0$, $(x_n,y_n)\rar (0,0)$. As we showed, the origin is non-Nash, so we are done.
\end{proof}

The previous proposition shows that a two-timescale version of simultaneous gradient ascent will still be susceptible to non-Nash equilibria. This implies that such an approach cannot guarantee convergence to only local Nash equilibria. We now show that the consensus optimization approach introduced in \cite{paper:CO}, can also converge to non-Nash points.

\begin{proposition}
	Consensus optimization can converge to non-Nash equilibria.
\end{proposition}

\begin{proof}
	The update in consensus optimization is given by:
		\[ z_{n+1}=z_n-\gamma \left( \omega(z_n)+\lambda J^T(z_n)\omega(z_n)\right),\]
where $\lambda>0$ is a hyperparameter to be chosen. We note that the signs are different than in \cite{paper:CO} because we consider simultaneous gradient descent as the base algorithm while they consider simultaneous gradient ascent. The limiting dynamics of this approach is given by:
	\[ \dot z = -\left( \omega(z_n)+\lambda J^T(z_n)\omega(z_n)\right).\]
	At a critical point $z$ where $\omega(z)=0$, the Jacobian of this dynamics is given by:
	\[ J_{CO}(z)=(J(z)+\lambda J^T(z)J(z)). \]
	Now, in the game described in Example~\ref{ex:counterex}, the Jacobian of the consensus optimization dynamics would be:
	\[ J_{CO}(x,y)=\bmat{1+2\lambda & 1+1.1\lambda \\- 1 +1.1\lambda & -0.1+1.01\lambda}.\]
	For $\lambda>0$, $J_{CO}(x,y)$ has eigenvalues with strictly positive real parts, which implies that, with any choice of the hyper-parameter $\lambda$, consensus optimization will converge to the non-Nash equilibrium at (0,0) in Example~\ref{ex:counterex}.
\end{proof}

The preceding proposition shows that the adjustment to the gradient dynamics proposed in \cite{paper:CO} does not solve the issue of non-Nash fixed points. This is not surprising since the primary goal for the algorithm is to reduce oscillation around equilibria and speed up convergence to the stable equilibria of the gradient dynamics. We note that as shown in Theorem~\ref{thm:algprops}, the proposed algorithm also achieves this goal. 

The final algorithm we consider is the symplectic gradient adjustment proposed in \cite{paper:ACO}. In the text, the authors do remark that the adjustment is not enough to guarantee convergence to only Nash equilibrium. For completeness, we make this assertion concrete by showing that the adjustment term is not enough to avoid the non-Nash equilibrium in Example~\ref{ex:counterex}.

\begin{proposition}
	The symplectic gradient adjustment to the gradient dynamics can converge to non-Nash equilibria.
\end{proposition}

\begin{proof}
	The dynamics resulting from the symplectic gradient adjustment is given by:
		\[ z_{n+1}=z_n-\gamma \left( \omega(z_n)+\frac{\lambda}{2} (J(z_n)-J^T(z_n))^T\omega(z_n)\right),\]
where $\lambda>0$ is a hyperparameter to be chosen. The limiting dynamics of this algorithm is given by:
	\[ \dot z = -\left( \omega(z_n)+\frac{\lambda}{2} (J(z_n)-J^T(z_n))^T\omega(z_n)\right).\]
	At a critical point $z$ where $\omega(z)=0$, the Jacobian of the above dynamics is given by:
	\[ J_{SGA}(z)=(J(z)+\frac{\lambda}{2}( J(z)-J^T(z))^TJ(z)). \]
	Now, in the game described in Example~\ref{ex:counterex}, the Jacobian of the SGA dynamics is given by:
	\[ J_{SGA}(x,y)=\bmat{1+\frac{\lambda}{2} & 1+0.1\frac{\lambda}{2} \\ 1 +\frac{\lambda}{2} & -0.1+\frac{\lambda}{2}}.\]
	This has eigenvalues with strictly  positive real parts for all values of $\lambda>0$, which implies that the symplectic gradient adjustment will converge to the non-Nash equilibrium at (0,0) in Example~\ref{ex:counterex}.
\end{proof}

%% file: jvps.tex
We now show that by leveraging auto-differentiation tools we can compute the stochastic approximation update efficiently. In particular, using Jacobian-vector products, the computation of $h_1$ and $h_2$ can be done relatively quickly. We first note that a Jacobian-vector product calculates $J^T(z)u$ for a constant vector $u \in \mb{R}^d$, by calculating the gradient of $\omega^T(z)u$ with respect to $v$. This allows us to write the $x$ and $y$ updates in Algorithm~\ref{alg:alg1} in a clean form.

The next proposition shows that, using the structure of $J$ that was discussed in Section~\ref{sec:prelims}, the term $J(z)v$ can also be efficiently computed. Together these results show that the per-iteration complexity of LSS is on the same order as that of consensus optimization and the symplectic gradient adjustment. 

\begin{proposition}
\label{prop:ctproperties}
The computation of $J(z)v$ requires two Jacobian-vector products, for $v\in \mb{R}^d$ a constant vector.
\end{proposition}

\begin{proof}
Let $v=(v_1,v_2)^T$ where $v_1\in \mb{R}^{d_x}$ and $v_2\in \mb{R}^{d_y}$. Then $J(z)v$ can be written as:
\[ J(z)v=\bmat{I & 0\\ 0 &-I}J^T(z)\bmat{v_1\\0} +\bmat{-I & 0\\ 0 &I}J^T(z)\bmat{0\\v_2}.\]
We note that the above expression is only possible due to the structure of the Jacobian in two-player zero-sum games. Having written $J(z)$ as above, it is now clear that computing $J(z)v$ will require two Jacobian-vector products.
\end{proof}

%% file: app_numex.tex
We now present further numerical experiments as well as the specific experimental setup we employed when training GANs. In Figure~\ref{fig:gan2} we show further numerical experiments that show the training of a generative adversarial network from various initial conditions. In Figure~\ref{fig:gan2}A, both simGD and LSS quickly converge to the correct solution. In Figure~\ref{fig:gan2}B, however, simGD only correctly identifies five of the Gaussians in the mixture while LSS converges to the correct solution.

\begin{figure}[h]
 % {fig:subfigex}
\center    
      \includegraphics[width=0.75\columnwidth]{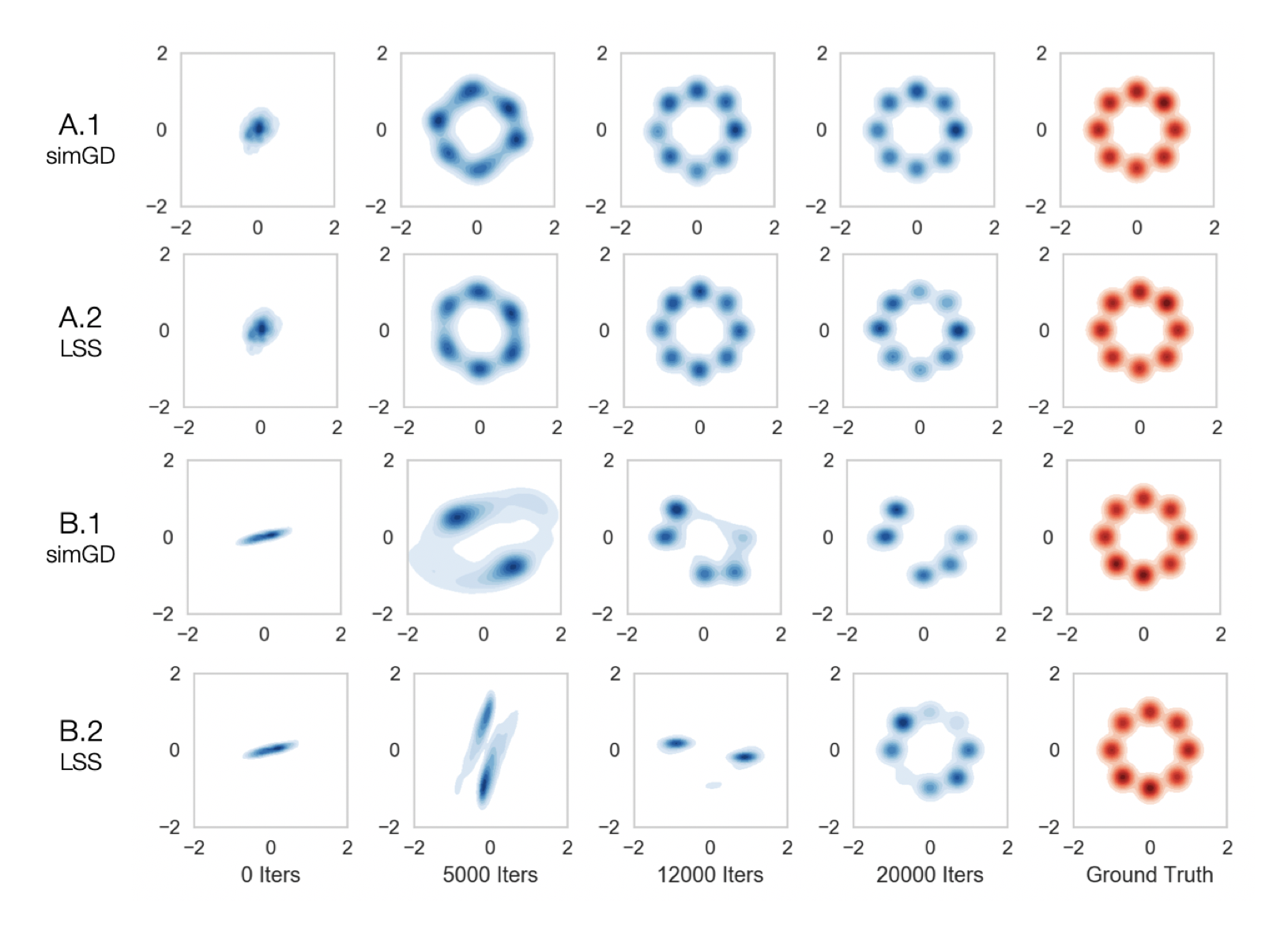}
        \caption{A. A GAN trained with (A.1) simGD and (A.2) with our two-timescale method. b. A GAN trained with (B.1) simGD and (B.2) with our two-timescale method. }
  \label{fig:gan2}
\end{figure}

For the training of the GAN, we randomly initialize the weights and biases using the standard TensorFlow initializations. For both the implementation of simGD and LSS used to generate Figure~\ref{fig:gan} and Figure~\ref{fig:gan2}, we used the RMSProp optimizer with step size $2e-4$ for the $x$ and $y$ processes. For the $v$ process in LSS, we used the RMSProp optimizer with step size $1e-5$. We note that these do not satisfy our assumptions on the step size sequences for the two-timescale process, but are meant to show that the approach still works in practical settings. Lastly, we chose a damping function $g(z)=e^{-0.001||\omega(z)||^2}$ and a batch size of $128$.

%% file: old_props.tex
% !TEX root = JMLR18.tex

In this section we present a useful theorem from previous work that we invoke in our proofs. The first proposition, a combination of Proposition 4.1 and 4.2 from \cite{benaim:1999ab}, gives us conditions under which a stochastic approximation process is an asymptotic pseudo-trajectory of the underlying ODE. 

\begin{proposition}
Let $h$ be a continuous globally integrable vector field. Further, let $x(t,s,x_s)$ for $t\ge s$ be a trajectory of the dynamical system $\dot x=-h(x)$ starting from state $x_s$ at time $s$. Finally let the stochastic approximation process be given by:
\[ x_{n+1}=x_n+a_n(h(x_n)+\chi_n+M_{n+1}),\]
where:
\begin{enumerate}
	\item $\sup_n ||x_n|| <\infty$
	\item $\sup_n \mb{E}[||M_{n}||^2<\infty$
	\item $\sum_{i=0}^\infty a_n^2<\infty$
	\item $\sum_{i=0}^\infty a_n=\infty$
	\item $\lim_{n\rar\infty} \chi_n=0$ almost surely.
\end{enumerate}
Then:
\[ \lim_{n\rar \infty} \sup_{0\le h \le K} \ ||x_{n+h}-x(t_{n+h},t_n, x_n)\|=0.\]
\end{proposition}